\documentclass{article}
\usepackage[latin2]{inputenc}
\usepackage{amssymb}
\usepackage{amsthm}
\usepackage{graphicx}
\usepackage[pagebackref = true]{hyperref}

\newtheorem{theorem}{Theorem}
\newtheorem{definition}[theorem]{Definition}

\title{Evolutionary solving of the debts' clearing problem}
\author{Csaba Pătcaş \and Attila Bartha}
\date{}

\begin{document}
\maketitle

\begin{abstract}
The debts' clearing problem is about clearing all the debts in a group of $n$ entities (persons, companies etc.) using a minimal number of money transaction operations. The problem is known to be NP-hard in the strong sense. As for many intractable problems, techniques from the field of artificial intelligence are useful in finding solutions close to optimum for large inputs. An evolutionary algorithm for solving the debts' clearing problem is proposed.

\end{abstract}

\section{Introduction}

The problem of debt clearing (DC problem) is one, that arises in real life situations as well. In a group of persons that know each other it is not uncommon to borrow some amount of money to an acquaintance for a period of time. This process is also happening among different banks, or even countries. As money transactions are time and money sensitive operations, it is desirable to clear the debts in a minimal number of money transaction operations.

The problem of settling debts was discussed by Verhoeff in 2004 (\cite{Ver04}).

Pătcaş \cite{Pat09} later re-discovered the problem and proposed it in 2008 at the qualification contest of the Romanian national team of informatics. The solution was described in \cite{Pat09} and the problem conjectured to be intractable, which was earlier proved in \cite{Ver04}. In \cite{Pat11} the problem in a dynamic setting is discussed and a new algorithm given, having superior speed on some cases compared to the one described in \cite{Pat09}.

\section{Stating the problem}

The problem statement is the following:

\emph{Let us consider a number of $n$ entities (persons, companies etc.), and a list of $m$ borrowings among these entities. A borrowing can be described by three parameters: the index of the borrower entity, the index of the lender entity and the amount of money that was lent. The task is to find a minimal list of money transactions that clears the debts formed among these $n$ entities as a result of the $m$ borrowings made.}

\begin{figure}[tb]
\begin{center}

List of borrowings:

\begin{tabular}{|c|c|c|}
\hline
Borrower&Lender&Amount of money\\
\hline
1&3&4\\
3&4&7\\
4&2&2\\
2&1&2\\
1&5&1\\
3&5&1\\
5&4&2\\
\hline
\end{tabular}

\medskip
Solution:

\begin{tabular}{|c|c|c|}
\hline
Sender&Reciever&Amount of money\\
\hline
1&4&3\\
3&4&4\\
\hline
\end{tabular}

\caption{Example for the DC problem}
\label{fig:ex}

\end{center}
\end{figure}

It is natural to model this problem using graph theory. Consider the following definitions.

\begin{definition}[\cite{Pat09}]
Let $G(V,A,W)$ be a directed, weighted multigraph without loops,
$|V|=n$, $|A|=m$, $W:A \rightarrow \mathbb{Z}$, where $V$ is the set
of vertices, $A$ is the set of arcs and $W$ is the weight function.
$G$ represents the borrowings made, so we will call it the
\textbf{borrowing graph}.
\end{definition}

The borrowing graph corresponding to the example in Figure \ref{fig:ex} is depicted in Figure
\ref{fig:borrowing}.

\begin{figure}[tb]
\begin{center}
\includegraphics[width=10cm]{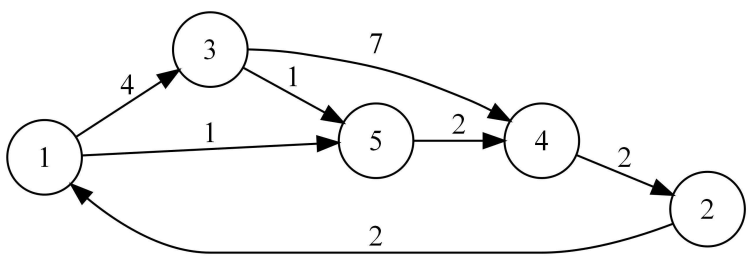}
\end{center}
\caption{The borrowing graph associated with the given example. An
arc from node $i$ to node $j$ with weight $w$ means, that entity $i$
must pay $w$ amount of money to entity $j$.}
\label{fig:borrowing}
\end{figure}


\begin{definition}[\cite{Pat09}]
Let us define for each vertex $v \in V$ the \textbf{absolute amount of debt} over the graph $G$: $D_G(v) = \sum
\limits_{\tiny\begin{array}{c}v' \in V\\ (v,v') \in A \end{array}}
W(v,v') - \sum \limits_{\tiny\begin{array}{c}v'' \in V\\ (v'',v) \in
A\end{array}} W(v'',v)$

Sometimes for simplicity we will refer to the absolute amount of debt of a node as \textbf{$D$ value}.
\end{definition}

The $D$ values corresponding to the example from Figure \ref{fig:ex} are listed in Figure \ref{fig:dvalues}.

\begin{figure}[tb]
\begin{center}
\begin{tabular}{|c|c|c|c|c|c|c|}
\hline
$i$&1&2&3&4&5\\
\hline
$D(i)$&3&0&4&-7&0\\
\hline
\end{tabular}

\caption{Absolute amounts of debt corresponding to the given example.}
\label{fig:dvalues}

\end{center}
\end{figure}

\begin{definition}[\cite{Pat09}]
Let $G'(V,A',W')$ be a directed, weighted multigraph without loops, with each arc $(i, j)$ representing a transaction of $W'(i, j)$ amount of money from entity $i$ to entity $j$. We call this graph a \textbf{transaction graph}. These transactions clear the debts formed by the borrowings modeled by graph $G(V,A,W)$ if and only if:

$D_G(v_i)=D_{G'}(v_i), \forall i=\overline{1,n}$, where $V=\{v_1,v_2, \ldots, v_n\}$

We will note this by: $G \sim G'$.
\end{definition}

See Figure \ref{fig:mintrans} for a transaction graph with minimal number of arcs corres\-ponding to the example from Figure \ref{fig:ex}.

\begin{figure}[tb]
\begin{center}
\includegraphics[width=10cm]{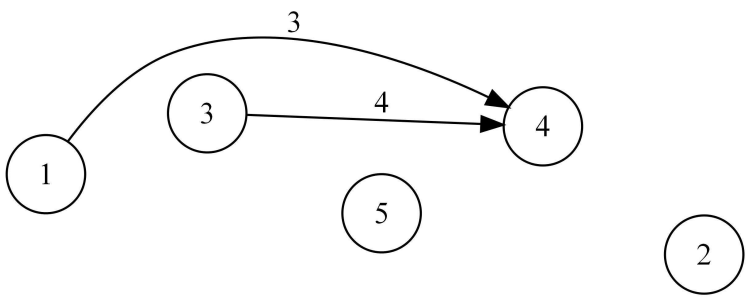}
\end{center}
\caption{The respective minimum transaction graph. An arc from node
$i$ to node $j$ with weight $w$ means, that entity $i$ pays $w$
amount of money to entity $j$.}
\label{fig:mintrans}
\end{figure}

We are now ready to reformulate the problem mathematically:

\emph{Given a borrowing graph $G(V,A,W)$ we are looking for a minimal tran\-sac\-tion graph $G_{min}(V,A_{min},W_{min})$, so that $G \sim G_{min}$ and $\forall G'(V,A',W') :$ $G\sim G', |A_{min}|\leq|A'|$ holds.}

\section{An equivalent problem}

\label{sec:equivalent}

The following observation is crucial in all of the solutions known so far.

\begin{theorem}
Any instance of the DC problem can be solved trivially by at most $n - 1$ transactions.
\label{the:trivial}
\end{theorem}

\begin{proof}
We give an algorithmic proof.

\begin{enumerate}
\item Let us choose two nodes $i$ and $j$, such that $D(i) > 0$ and $D(j) < 0$.
\item Add arc $(i, j)$ to the transaction graph having weight $\min(D(i), -D(j))$.
\item Update the $D$ values of $i$ and $j$ to reflect the addition of the arc (by decreasing $D(i)$ and increasing $D(j)$).
\item Repeat steps (1) - (3) as long as possible.
\end{enumerate}

It is clear that at least one $D$ value becomes zero as a result of executing steps (1) - (3). Also, because we have the invariant that the sum of all $D$ values is always zero, at the last iteration we always have $D(i) = -D(j)$. Thus two $D$ values become zero at the last iteration, which yields to the needed upper bound.
\end{proof}

We observe, that finding a minimal transaction graph is equivalent to partitioning $V$ into a maximal number of disjoint zero-sum subsets, more formally $V = P_1 \cup \ldots \cup P_{max}$, $\sum \limits_{u \in P_i} D(u) = 0, \forall i = \overline{1, max}$ and $P_i \cap P_j = \emptyset, \forall i, j = \overline{1, max}, i \neq j$. The reason for this is, that all the debts in a zero-sum subset $P_i$ can be cleared by $|P_i| - 1$ transactions by Theorem \ref{the:trivial}, thus to clear all the debts, $|V| - max$ transactions are necessary.

\section{Evolutionary technique for solving the DC problem}

We use the reformulation of the problem described in Section \ref{sec:equivalent}.

\paragraph{Representation}

A solution of the problem is represented by a permutation of the $D$ values of $V$, the set of nodes. Thus a candidate solution is a vector $C = (c_1, c_2, \ldots, c_{n})$, such that $c_i = D(u), \forall i \in \overline{1, n}$ for some unique $u \in V$.

For instance $C = (3, 0, -7, 4, 0)$ is a chromosome representing a candidate solution for the $D$ values from Figure \ref{fig:dvalues}.

The idea of permutation representations is used intensively in solutions of the Traveling Salesman Problem (\cite{GolLin85, OliSmiHol87, WhiStaFuq89}).

\paragraph{Fitness assignment}

To evaluate the fitness of a chromosome, we iterate over the genes of the chromosome in increasing order and maintain the partial sum obtained so far, that is $s_i = \sum \limits_{j = 1}^{i} c_j$. For every $s_i = 0$, we have found a new zero-sum subset of the partition (starting after the last encountered partial sum equal to zero and ending at $i$), so we can add one to the fitness of the chromosome.

For instance if we have $C = (-3, 2, 1, -5, 5)$, then $s = (-3, -1, 0, -5, 0)$, so the fitness of $C$ will be 2, corresponding to the partition formed by the first three elements and the last two elements.

\paragraph{Recombination}

Various operators for permutation representations are discussed in \cite{Dav85, GolLin85, Gor90, OliSmiHol87, Sys91, WhiStaFuq89}. We propose new recombination operators.

\subparagraph{Recomb1}

Let $C_1$ and $C_2$ be the two chromosomes, and $k \in [1, n]$ a random crossover point. Then, the first descendant $C_1'$  can be obtained by copying the first $k$ genes from $C_1$ and appending to it the elements of the permutation not used so far in the same order as they appear in $C_2$. The second descendant $C_2'$ is obtained symmetrically.

For instance,

\begin{center}
$k = 2$

$C_1 = (\mathbf{-3, 2}, 1, -5, 5)$ $C_2 = (\mathbf{-5, 2}, 1, -3, 5)$

$\downarrow$

$C_1' = (\mathbf{-3, 2}, -5, 1, 5)$ $C_2' = (\mathbf{-5, 2}, -3, 1, 5)$

\end{center}

\subparagraph{Recomb2}

The problem with \texttt{Recomb1} is, that the first descendant inherits most of its properties from $C_1$ and very little from $C_2$. Symmetrically $C_2'$ inherits most of its properties from $C_2$ and very little from $C_1$. This is undesirable, as both $C_1$ and $C_2$ can contain subsets from the optimal partition.

A better recombination operator may be the following. First, determine the partitions codified by $C_1$ and $C_2$, as described at the evaluation of the fitness function. Let those be $C_1 = P_{1, 1} \cup P_{1, 2} \cup \ldots$ and $C_2 = P_{2, 1} \cup P_{2, 2} \cup \ldots$. Initialize $C_1' := C_1$ and $C_2' := C_2$.

Then, iterate over every $P_{1, i}$. If some $P_{1, i}$ is contained in some $P_{2, j}$, that is $P_{1, i} \subset P_{2, j}$, replace $P_{2, j}$ in the second descendant with $P_{1, i} \cup (P_{2, j} \setminus P_{1, i})$. Repeat the same procedure for $C_2$ symmetrically.

For instance,

\begin{center}
$C_1 = (-3, 2, 1, -5, 5) = \{-3, 2, 1\} \cup \{-5, 5\}$

$C_2 = (2, 1, 5, -5, -3) = \{2, 1, 5, -5, -3\}$

$\downarrow$

$C_1' = \{-3, 2, 1\} \cup \{-5, 5\} = (-3, 2, 1, -5, 5)$

$C_2' = \{-3, 2, 1\} \cup \{5, -5\} = (-3, 2, 1, 5, -5)$

\end{center}

\paragraph{Mutation}

Two new mutation operators are proposed, having the property, that the fitness of the chromosome does not decrease.

\subparagraph{Mut1}

The inversion operator described by Holland (\cite{Hol75}) can be used without modification, on the sequence between the $i^{th}$ and $j^{th}$ elements.

For instance,

\begin{center}
$i = 2, j = 5$

$C = (-3, \mathbf{2, 1, -5, 5})$

$\downarrow$

$C' = (-3, \mathbf{5, -5, 1, 2})$

\end{center}

\subparagraph{Mut2}

\texttt{Mut1} can be used on the partition $C = P_1 \cup P_2 \cup \ldots$ instead of the permutation representation. This method guarantees that the fitness of the chromosome does not decrease.

For instance,

\begin{center}
$i = 1, j = 4$

$C = (-2, 2, 3, 4, -7, 1, -1, 6, -3, 2, -5) = \mathbf{\{-2, 2\} \cup \{3, 4, -7\} \cup \{1, -1\} \cup \{6, -3, 2, -5\}}$

$\downarrow$

$C' =  \mathbf{\{6, -3, 2, -5\} \cup \{1, -1\} \cup \{3, 4, -7\} \cup \{-2, 2\}} = (6, -3, 2, -5, 1, -1, 3, 4, -7, -2, 2)$

\end{center}

\subparagraph{Mut3}

\texttt{Mut1} can also be used inside some $P_k$ without decreasing the fitness.

For instance,

\begin{center}
$k = 4, i = 1, j = 4$

$C = (-2, 2, 3, 4, -7, 1, -1, 6, -3, 2, -5) = \{-2, 2\} \cup \{3, 4, -7\} \cup \{1, -1\} \cup \mathbf{\{6, -3, 2, -5\}}$

$\downarrow$

$C' = \{-2, 2\} \cup \{3, 4, -7\} \cup \{1, -1\} \cup \mathbf{\{-5, 2, -3, 6\}} = (-2, 2, 3, 4, -7, 1, -1, -5, 2, -3, 6)$
\end{center}

\section{How to obtain large instances of the DC problem}

Because of the strongly NP-hardness of the problem, it is challenging to generate large test cases for which information about the optimal solution is known. We describe five methods to generate large test cases.

\paragraph{Method 1}

If the optimal solution for some input is known, padding the set of $D$ values with $k$ zeros increases the optimal solution also by $k$.

\paragraph{Method 2}

Method 1 can be modified by padding the input with $k$ pairs of the structure $(x, -x)$.

\paragraph{Method 3}

If the number of negative (or positive) numbers is two, the problem is equivalent to the Subset Sum problem and is solvable in pseudopolynomial time by dynamic programming. Using this method we can generate inputs for which the optimal solution is unique, that is, there is a single subset of positive (negative) numbers having the sum equal to one of the two negative (positive) numbers (in absolute value). An optimal answer for such an input is expected to be difficult to find for our evolutionary approach, as in the worst case (when the cardinality of the subset is $n / 2$) only $2 \cdot (\frac{n}{2}!)^2$ out of the $n!$ possible permutations do represent an optimal solution. For $n = 10$, this means that the ratio of optimal solutions and all solutions is about $7.9 \cdot 10^{-3}$, while for $n = 100$ the ratio is about $1.9 \cdot 10^{-29}$.

This idea can be extended for any fixed number of negative (positive) numbers, but the running time of the dynamic programming solution raises quickly.

\paragraph{Method 4}

Let $n$ be the desired size of the input and $l \leq \lfloor n / 2 \rfloor$ an integer. First generate randomly a set of $n - l$ elements, containing only positive $D$ values and $l$ distinct integers from the $[1, n - l]$ range (denoted $r_1 < \ldots < r_l$). Let $s$ be the vector of partial sums, that is $s_i = \sum \limits_{j = 1}^i D(j), \forall i = \overline{1, n - l}$ (we assume $s_0 = 0$ and $r_0 = 0$). For every $r_i, \forall i = \overline{1, l}$ insert $- (s_{r_i} - s_{r_{i - 1}})$ to the set. In other words we insert with a negative sign the sum of $l$ partial sequences, whose borders are denoted by $r_{i - 1}$ and $r_{i}$. By this method we can get the optimal solution to be equal to $l$. As the range of the possible values of the first $n - l$ positive elements gets bigger, we expect the optimal solution to be harder and harder to find. The reason is, that the probability to get the same sum from a different combination of positive numbers gets smaller, thus the number of genetic representations corresponding to an optimal solution decreases.

\paragraph{Method 5}

It can be easily seen, that if the optimal solution for a set $V$ is known to be $max$, then the solution for $V \cup V$ will be $2 \cdot max$, the solution for $V \cup V \cup V$ will be $3 \cdot max$ and so on.

\section{Numerical experiments}

A preliminary testing phase was carried out using the same 15 test cases which were used when the problem was proposed in 2008 at the qualification contest of the Romanian national team (see \cite{Pat11}). These test cases all have specially crafted structures, with $n \leq 20$, $m \leq 100$ and the cost of an arc being a natural number no larger than 100. The optimal solution was found for each test case by using the algorithm described in \cite{Pat09}. Our genetic algorithm found the optimal solution for all the test cases.

\subsection{Combinations of operators}

In the first set of experiments our goal was to determine which combinations of our recombination and mutation operators work best in practice, along with desirable values for mutation probability. We constructed three test cases (\texttt{debt100a}, \texttt{debt100b} and \texttt{debt100c}) with different structures, all of them having $n = 100$.

\texttt{debt100a} was obtained by concatenating the test case from the initial 15 that was the most difficult to solve for the genetic algorithm (case 15) five times to itself. By the observation above in Method 5, the optimal solution for this test case is $max = 25$.

To generate \texttt{debt100b} we used Method 3 for $n = 50$ and concatenated the obtained set once to itself, thus obtaining a case having $max = 4$ by the observation above.

To obtain \texttt{debt100c} we first generated using a dynamic programming algorithm a set having 20 elements, which can be uniquely partitioned into three zero-sum subsets (and no more). Then we concatenated this set five times to itself, yielding $max = 15$ for this test case.

For each of the three described test cases we used the following methodology. For every possible combination of recombination and mutation operators we fixed the mutation probability to every value from 0 to 1 in steps of 0.1 and executed the genetic algorithm 10 times. We recorded the best solution obtained among the 10 executions, the average of the 10 best values and the average fitness of all genomes. In each case the population size was set to 100 individuals and the number of generations to 1000. The best five individuals always survived to the next generation.

To assess the efficacy of our algorithm we compared it to an algorithm called \texttt{RandomSearch}, which works by generating an independent random solution in every generation for every chromosome. In our case this meant generating 100000 random solutions and remembering the one with the maximum fitness value among them.

The results of the first set of experiments were the following:

\begin{itemize}

\item \texttt{debt100c} was the most difficult of the three test cases used, no algorithm being able to find the optimal solution $max = 15$. The best solution found by \texttt{RandomSearch} was 5, and the best solution found by the evolutionary algorithms was 13, using \texttt{Recomb2} along with \texttt{Mut1} with a mutation probability ranging from 0.8 to 1. The average fitness of all genomes was maximal at mutation probability 0.7.

\item \texttt{debt100b} was the easiest of the test cases, our genetic algorithm being able to find the optimal solution $max = 4$ in the majority of the cases (in about 76\% of the possible combinations of recombination and mutation operators and mutation probabilities). Mutation probability 0.7 along with \texttt{Recomb2} and \texttt{Mut1} maximized the average fitness again. The best solution found by \texttt{RandomSearch} was 3.

\item For \texttt{debt100a} \texttt{RandomSearch} was able to find a solution with fitness 9. Our genetic algorithm found the optimal solution 25 in a small percentage of the cases, using the same parameters that yielded the best solutions for \texttt{debt100c}. Maximal average fitness was obtained with mutation probability 0.4 using \texttt{Recomb2} and \texttt{Mut1}.

\end{itemize}

We can draw the conclusion, that our genetic algorithm is much more efficient than generating random solutions. The results suggest, that using \texttt{Recomb2} with \texttt{Mut1} works best in practice for a wide range of inputs. On the other hand we note, that \texttt{Recomb2} and \texttt{Mut2} is a particularly bad combination, the reason being, that it does not allow the exploration of a sufficient varied range of solutions, because neither of the operators is able to introduce new partition sets into the population. Still, \texttt{Mut2} works fairly well together with \texttt{Recomb1}, as the latter is capable of constructing new partition sets.

\subsection{Convergence to optimum}

In the second set of experiments we studied the convergence of the solution to the optimal value as the number of generations increases. We concatenated each of the three test cases described above ten times to itself, obtaining cases \texttt{debt1000a}, \texttt{debt1000b} and \texttt{debt1000c} respectively. We executed our genetic algorithm using \texttt{Recomb2} and \texttt{Mut1} with a mutation probability of 0.75. The population size was set to 80 and the best five individuals were always promoted to the next generation. The algorithm was executed once for 50000 generations, and the fitness of the best chromosome was recorded every 100 generations.

The results are depicted in Figures \ref{fig:debt1000a}, \ref{fig:debt1000b} and \ref{fig:debt1000c}. We can observe that in every case the fitness of the best individual raises sharply in the first 5000 generations, then slows down gradually. 50000 generations were enough to find a solution having fitness 244 (97.6\% of the optimum) for \texttt{debt1000a} and a solution having fitness 39 (97.5\% of the optimum) for \texttt{debt1000b}. Case \texttt{debt1000c} was significantly more difficult, the best solution having only fitness 122 (81.3\% of the optimum).

\begin{figure}[tb]
\begin{center}
\includegraphics[width=10cm]{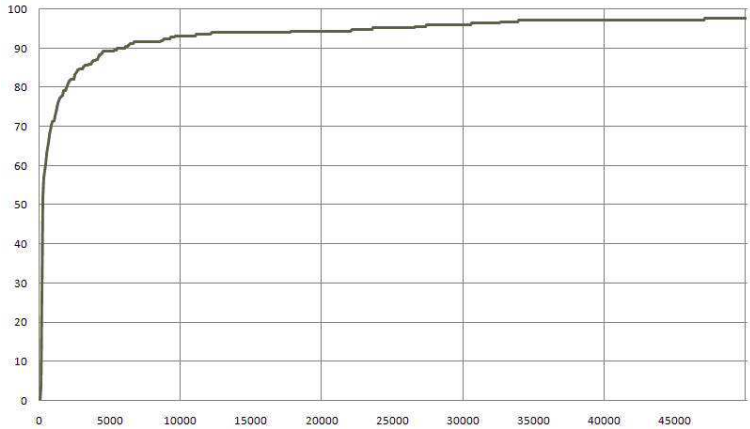}
\end{center}

\caption{The fitness of the best individual compared to the optimal solution in percentages for test case \texttt{debt1000a} as the number of generations increases.}
\label{fig:debt1000a}
\end{figure}

\begin{figure}[tb]
\begin{center}
\includegraphics[width=10cm]{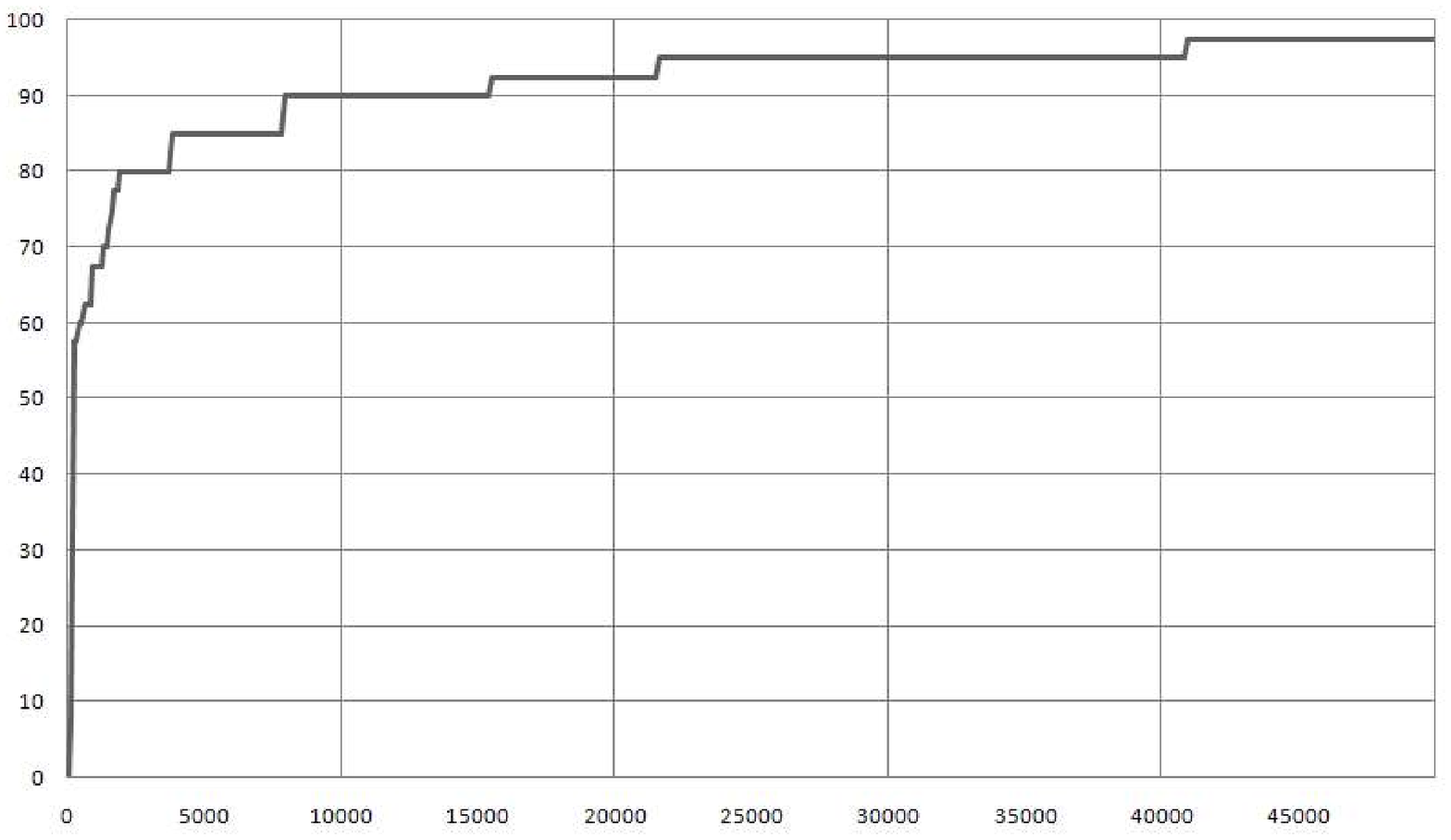}
\end{center}

\caption{The fitness of the best individual compared to the optimal solution in percentages for test case \texttt{debt1000b} as the number of generations increases.}
\label{fig:debt1000b}
\end{figure}

\begin{figure}[tb]
\begin{center}
\includegraphics[width=10cm]{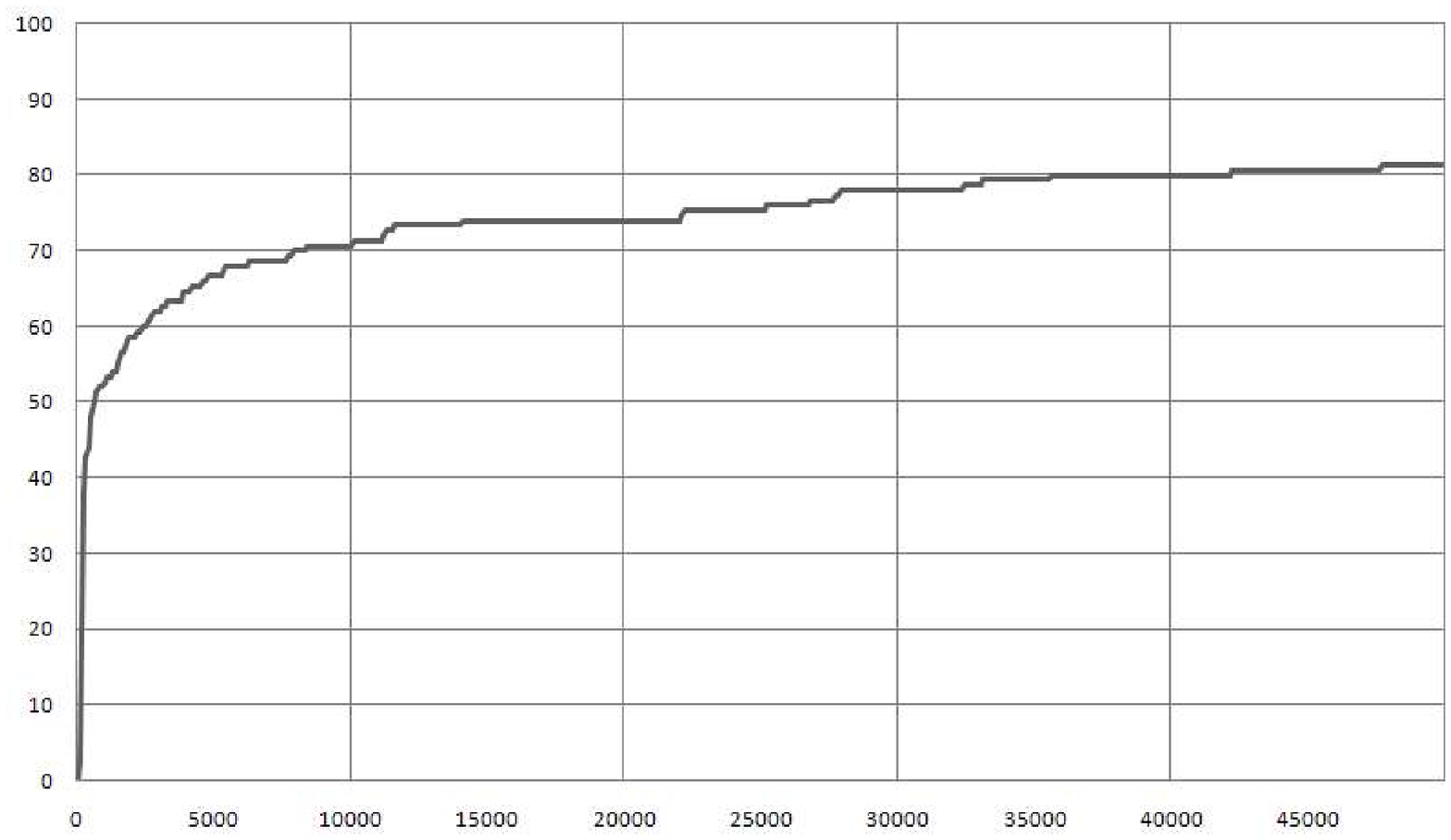}
\end{center}

\caption{The fitness of the best individual compared to the optimal solution in percentages for test case \texttt{debt1000c} as the number of generations increases.}
\label{fig:debt1000c}
\end{figure}

\subsection{Efficiency on very difficult test cases}

In the third set of experiments we used Method 2 to generate test cases, which are very difficult for our evolutionary algorithm. Starting with $n = 100$ and going by increments of 100 we generated sets having the structure $\{1, 2, \ldots, n / 2, \\-1, -2, \ldots, - n / 2\}$. It can be easily seen, that the optimal solution for these cases is $max = n / 2$ and it is unique. Only $\frac{n}{2}! \cdot 2^{n / 2}$ representations out of $n!$ translate to an optimal solution, which means that the ratio of optimal solutions to all solutions is about $1.0 \cdot 10^{-3}$ for $n = 10$ and about $3.6 \cdot 10^{-79}$ for $n = 100$.

For every case we executed the genetic algorithm 10 times using \texttt{Recomb2} and \texttt{Mut1} with a mutation probability 0.75. The population size was set to 80 and the best five individuals were always promoted to the next generation. The algorithm was stopped after 5000 generations. For every test case we recorded the best solution found by the algorithm, the average of the best solutions over the 10 executions and the summed up running time of the 10 executions. The results are presented in Figure \ref{fig:exp3}.

\begin{figure}[tb]
\begin{center}

\begin{tabular}{|c|c|c|c|}
\hline
N&Best solution&Average of bests&Running time\\
&(\% of optimum)&(\% of optimum)&(in seconds)\\
\hline
100&50 (100\%)&50 (100\%)&203\\
200&76 (76\%)&70.4 (70.4\%)&710\\
300&91 (60.6\%)&85.5 (57\%)&1247\\
400&108 (54\%)&100.5 (50.2\%)&1919\\
500&116 (46.4\%)&109.8 (43.9\%)&2610\\
600&130 (43.3\%)&121.1 (40.3\%)&3328\\
700&138 (39.4\%)&132 (37.7\%)&4225\\
800&147 (36.7\%)&142.4 (35.6\%)&5134\\
900&155 (34.4\%)&146.6 (32.5\%)&6084\\
1000&166 (33.2\%)&157.3 (31.4\%)&6766\\
\hline
\end{tabular}

\end{center}

\caption{Results of 10 executions for 5000 generations each, on very difficult test cases}
\label{fig:exp3}
\end{figure}

For $n = 100$ the optimal solution was found in all of the 10 executions, but as the size of the input increased, the best solution got further and further from the optimum. We note the robustness of the algorithm, as the best solution is usually just a few percentages away from the average.

\section{Conclusions}

The debts' clearing problem is an NP-hard problem of practical interest, as it arises in real life situations as well. The only known algorithms to solve the problem were the ones presented in \cite{Pat09} and \cite{Pat11}, which are exact algorithms that provide the optimal solution always, but their running time is practical only for small inputs ($n \leq 20$).

Using an equivalent problem we described an evolutionary algorithm to solve the problem and made extensive experiments to assess its efficacy. From the experiments we concluded, that our algorithm is much more efficient than a random search in the space of the solutions. Our algorithm is capable of finding the optimal solution for the most difficult test cases with sizes up to $n = 100$ in a matter of minutes. For cases as large as $n = 1000$ our approach remains practical, as it can obtain solutions in the range of 80\% - 98\% compared to the optimal solution in about an hour on a personal computer. In comparison a random search does not go above 15\% even for the easiest cases of this size.

\end{document}